\documentclass{article}

\usepackage[table,xcdraw]{xcolor}
\usepackage[acronym]{glossaries}
\usepackage{hyperref}
\usepackage{url}
\usepackage{xifthen}
\usepackage{makecell}
\usepackage{booktabs}
\usepackage{amsfonts}
\usepackage{nicefrac}
\usepackage{microtype}
\usepackage{amsmath}
\usepackage{multirow}
\usepackage{subcaption}
\usepackage{mathtools}
\usepackage{amsthm}
\usepackage[capitalize,noabbrev]{cleveref}
\usepackage{graphicx}
\usepackage[makeroom]{cancel}

\usepackage[accepted]{icml2025}

\def\N{\mathcal{N}}
\def\E{\mathbb{E}}
\def\R{\mathbb{R}}
\def\O{\mathcal{O}}
\def\X{\mathbf{X}}
\def\setX{\mathrm{X}}
\def\Z{\mathbf{Z}}
\def\w{\mathbf{w}}
\def\bw{\bar{\w}}
\def\a{\alpha}
\def\s{\sigma}
\def\d{\delta}
\def\e{\varepsilon}
\def\ph{\varphi}
\def\th{\theta}
\def\barf{\bar{f}}
\def\barh{\bar{h}}

\def\p{\partial}
\def\n{\nabla}
\def\KL{{\mathrm{D}_\mathrm{KL}}}
\def\L{\mathcal{L}}
\def\Lrec{{\L_\mathrm{rec}}}
\def\Ldiff{{\L_\mathrm{diff}}}
\def\Lprior{{\L_\mathrm{prior}}}

\newcommand{\TODO}[1][]{{\ifthenelse{\isempty{#1}}{\color{red}\bf TODO}{\color{red}\bf (TODO: #1)}}}

\theoremstyle{plain}
\newtheorem{theorem}{Theorem}[section]
\newtheorem{proposition}[theorem]{Proposition}

\theoremstyle{definition}

\theoremstyle{remark}

\makeglossaries
\newglossaryentry{simfree}{
    name={simulation-free},
    description={}
}
\newglossaryentry{simbased}{
    name={simulation-based},
    description={}
}
\newglossaryentry{prproc}{
    name={prior process},
    plural={prior processes},
    description={}
}
\newglossaryentry{ptproc}{
    name={posterior process},
    plural={posterior processes},
    description={}
}
\newglossaryentry{smc}{
    name={Sequential Monte Carlo},
    description={}
}
\newacronym{sde}{SDE}{Stochastic Differential Equation}
\newacronym{ode}{ODE}{Ordinary Differential Equation}
\newacronym{elbo}{ELBO}{Evidence Lower Bound}
\newacronym{nelbo}{NELBO}{Negative Evidence Lower Bound}
\newacronym{nll}{NLL}{negative log-likelihood}
\newacronym{kl}{KL}{Kullback–Leibler}
\newacronym{vjp}{VJP}{Vector Jacobian Product}
\newacronym{jvp}{JVP}{Jacobian Vector Product}
\newacronym{vae}{VAE}{Variational Autoencoder}

\usepackage[textsize=tiny]{todonotes}

\icmltitlerunning{SDE Matching: Scalable and Simulation-Free Training of Latent SDEs}

\begin{document}

\twocolumn[
\icmltitle{SDE Matching: Scalable and Simulation-Free Training\\
of Latent Stochastic Differential Equations}



\icmlsetsymbol{equal}{*}

\begin{icmlauthorlist}
\icmlauthor{Grigory Bartosh}{uva}
\icmlauthor{Dmitry Vetrov}{ctr}
\icmlauthor{Christian A. Naesseth}{uva}
\end{icmlauthorlist}

\icmlaffiliation{uva}{University of Amsterdam}
\icmlaffiliation{ctr}{Constructor University, Bremen}

\icmlcorrespondingauthor{Grigory Bartosh}{g.bartosh@uva.nl}
\icmlcorrespondingauthor{Dmitry Vetrov}{dvetrov@constructor.university}
\icmlcorrespondingauthor{Christian A. Naesseth}{c.a.naesseth@uva.nl}

\icmlkeywords{diffusion, generative models, variational inference}

\vskip 0.3in
]



\printAffiliationsAndNotice{}  

\begin{abstract}
The Latent Stochastic Differential Equation (SDE) is a powerful tool for time series and sequence modeling. However, training Latent SDEs typically relies on adjoint sensitivity methods, which depend on simulation and backpropagation through approximate SDE solutions, which limit scalability. In this work, we propose SDE Matching, a new simulation-free method for training Latent SDEs. Inspired by modern Score- and Flow Matching algorithms for learning generative dynamics, we extend these ideas to the domain of stochastic dynamics for time series and sequence modeling, eliminating the need for costly numerical simulations. Our results demonstrate that SDE Matching achieves performance comparable to adjoint sensitivity methods while drastically reducing computational complexity.
\end{abstract}

\glsresetall

\section{Introduction}
\label{sec:introduction}

\glsunsetall
\begin{table}[th]
\caption{Asymptotic complexity comparison. $L$ and $R$ are number of sequential and parallel evaluations of drift/diffusion terms, respectively, and $D$ is the number of parameters/states.}
\label{tab:complexity}
\centering
\setlength\extrarowheight{-7pt}
\begin{tabular}{lcc}
\toprule
    \textbf{Method} & \textbf{Memory} & \textbf{Time} \\ \midrule
    \makecell[l]{
        Forward Pathwise \\
        \citep{yang1991monte} \\
        \citep{gobet2005sensitivity}
    } & $\O(1)$ & $\O(LD)$ \\ \\
    \makecell[l]{
        Backprop through Solver \\
        \citep{giles2006smoking}
    } & $\O(L)$ & $\O(L)$ \\ \\
    \makecell[l]{
        Stochastic Adjoint \\
        \citep{li2020scalable}
    } & $\O(1)$ & $\O(L\log L)$ \\ \\
    \makecell[l]{
        Amortized Reparameterization \\
        \citep{course2024amortized}
    } & $\O(R)$ & $\O(R)$ \\ \\
    \makecell[l]{
        \textbf{\gls{sde} Matching} \\
        (this paper)
    } & $\O(1)$ & $\O(1)$
    \\
\bottomrule
\end{tabular}
\end{table}
\glsresetall

Differential equations are a natural choice for modeling continuous-time dynamical systems and have recently received significant interest in machine learning. Since \citet{chen2018neural} introduced the adjoint sensitivity method for learning \glspl{ode} in a memory-efficient manner, \gls{ode}-based approaches became popular in deep learning for density estimation \cite{grathwohl2018scalable} and to model irregularly observed time series \cite{yildiz2019ode2vae, rubanova2019latent}.

However, \glspl{ode} describe deterministic systems and encode all uncertainty into their initial conditions. This limits the applicability of \gls{ode}-based approaches when modeling stochastic and chaotic processes. To address these limitations, several works study Neural (or Latent) \glspl{sde} \citep{movellan2002monte,ha2018adaptive,tzen2019neural,peluchetti2020infinitely,hodgkinson2020stochastic}. Latent \glspl{sde} have been shown to be more robust to data shifts \citep{oh2024stable}, and applications are numerous, from neuroscience \citep{elgazzar2024generative} to video prediction \citep{daems2024variational}.

To enable more memory-efficient learning, \citet{li2020scalable} extended the adjoint sensitivity method from \glspl{ode} to \glspl{sde}.
Despite these advancements, training of \gls{ode} and \gls{sde} models is \emph{\gls{simbased}}, it relies on costly numerical integration of differential equations and backpropagation through the solutions. With training algorithms that are difficult to parallelize on modern hardware, Latent \glspl{sde} have till now resisted truly scaling.

In parallel, Score Matching \citep{ho2020denoising, song2021scorebased} and Flow Matching methods \citep{lipman2023flow, albergo2023stochastic, liu2023flow} have demonstrated that continuous-time dynamics for generative modeling can be learned efficiently in a \emph{\gls{simfree}} manner—without requiring numerical integration. These techniques have proven computationally efficient and scalable for high-dimensional problems. Inspired by these developments, we develop a \gls{simfree} approach for learning \glspl{sde}.

We introduce \gls{sde} Matching—a \gls{simfree} framework for learning Latent \gls{sde} models. The key idea we adopt from matching-based approaches is direct access to samples from the latent \emph{posterior} approximation at any time step. This eliminates the need to integrate \glspl{sde} during training. The \gls{sde} Matching objective is estimated using the Monte Carlo method, achieving $\O(1)$ memory and time complexity. We summarize the complexity comparisons in \cref{tab:complexity}.

We summarize our contributions as follows:
\begin{enumerate}
    \item We establish a clear connection between diffusion models and Latent \glspl{sde}, which motivates the development of a more efficient, simulation-free training procedure.

    \item We propose an efficient parameterization of the Latent \gls{sde} \gls{ptproc} and use it to develop \gls{sde} Matching, a \gls{simfree} training procedure for Latent and Neural \glspl{sde}.

    \item We demonstrate that \gls{sde} Matching achieves comparable performance to the adjoint sensitivity method on both synthetic and real data, while significantly reducing the computational cost of training and improving convergence. 

    \item We show that \gls{sde} Matching enables the application of Latent \glspl{sde} to high-dimensional problems, where training with the adjoint sensitivity method is computationally infeasible.
\end{enumerate}

In \cref{sec:background}, we provide background on the Latent \gls{sde} model. In \cref{sec:diffusion}, we discuss the connection between Latent \glspl{sde} and matching-based methods, which serve as motivation for \gls{sde} Matching. Then, in \cref{sec:method}, we introduce the parameterization of the Latent \gls{sde} \gls{ptproc} and formulate the \gls{sde} Matching training algorithm. In \cref{sec:experiments}, we demonstrate that our method achieves comparable performance on both synthetic and real-world data while requiring significantly less computational and memory cost for training. In \cref{sec:related_work} we discuss and contrast \gls{sde} Matching to related work. Finally, we conclude with a discussion of the limitations of our approach and potential directions for future work.

\section{Background}
\label{sec:background}

Consider a series of observations ${\X = {x_{t_1}, x_{t_2}, \dots, x_{t_N}}}$, where each $x_{t_i} \in \setX$ where the space $\setX$ depends on the application. For simplicity of notation, we assume that all time steps $t_i$ belong to the interval $[0, 1]$. The extension to intervals of arbitrary lengths is straightforward.

The Latent \gls{sde} model assumes that this series is generated constructively by the following stochastic process, referred to as the \gls{prproc}. First, sample a latent variable ${z_0 \in \R^D}$ from the initial prior distribution $p_\th(z_0)$. Next, infer the latent continuous dynamics $\Z = {(z(t))}_{t \in [0, 1]}$ by integrating the following \gls{sde}: 
\begin{align}
    d z_t &= h_\th(z_t, t) d t + g_\th(z_t, t) d \w,
    \label{eq:background_prior_sde}
\end{align}
where $h_\th : \R^D \times [0, 1] \mapsto \R^D$ is the drift term, $g_\th(z_t, t): \R^D \times [0, 1] \mapsto \R^{D \times D}$ is the diffusion term, and $\w$ is a standard Wiener process. The \gls{sde} in \cref{eq:background_prior_sde}, together with the prior distribution $p_\th(z_0)$, defines a sequence of probability distributions $p_\th(z_t)$ at each time step $t \in [0, 1]$. Once the trajectory of the latent variables $\Z$ is sampled, we independently sample observations $x_{t_i}$ from the conditional distributions $p_\th(x_{t_i}|z_{t_i})$ for each time step $t_i$.

The goal of the Latent \gls{sde} is to find a set of parameters $\th$ that best fits a dataset of observed time series or sequences. Unfortunately, the posterior distribution of the latent variables, $p_\th(\Z | \X)$, is generally intractable. However, variational inference can be used to train the Latent \gls{sde}. Similar to the \gls{prproc}, we introduce an approximate \gls{ptproc}, which is conditioned on the observations $\X$. The \gls{ptproc} consists of two components: the initial posterior distribution $q_\ph(z_0|\X)$ and a conditional \gls{sde}: 
\begin{align}
    d z_t &= f_\ph(z_t, t, \X) d t + g_\th(z_t, t) d \w,
    \label{eq:background_posterior_sde}
\end{align}
where $f_\ph(z_t, t, \X)$ defines the drift term of the \gls{sde} conditionally on the observations $\X$. It is important to note that, despite having a different drift term, the posterior shares the same diffusion term $g_\th(z_t, t)$ as in \cref{eq:background_prior_sde}.

The training objective of the Latent \gls{sde} is a variational bound on the log-marginal likelihood of the observations: 
\begin{align}
    -\log p_\th(\X) & \leq \L = \Lprior + \Ldiff + \Lrec 
    \label{eq:background_loss} \\
    \Lprior & = \KL \big( q_\ph(z_0|\X) \| p_\th(z_0) \big) \\
    \Ldiff & = \underset{q_\ph(\Z|\X)}{\E} \left[ \int_0^1 \frac{1}{2} \left\| r_{\th,\ph}(z_t, t, \X) \right\|_2^2 dt \right] 
    \label{eq:background_l_diff} \\
    \Lrec & = \underset{q_\ph(\Z|\X)}{\E} \left[ \sum_{i = 1}^N - \log p_\th(x_{t_i}|z_{t_i}) \right],
\end{align}
where $r_{\th,\ph}(z_t, t, \X)$ satisfies
\begin{align}
    g_\th(z_t, t) r_{\th,\ph}(z_t, t, \X) = h_\th(z_t, t) -  f_\ph(z_t, t, \X).
\end{align}

Obtaining an estimate of the objective $\L$ for stochastic optimization requires joint numerical integration of the \gls{ptproc} \gls{sde}~(\cref{eq:background_posterior_sde})  and the integral in \cref{eq:background_l_diff}. While there exist methods to improve the memory and computational efficiency of numerical integration, most are \gls{simbased} and require backpropagation through numerical solutions of \glspl{sde}. Besides being computationally expensive, these methods are difficult to parallelize and suffer from numerical instabilities. Altogether, these challenges make training the Latent \gls{sde} a difficult task.

\section{Diffusion Models as Latent \glspl{sde}}
\label{sec:diffusion}

In contrast to Latent \glspl{sde}, recent advancements in diffusion and flow-based modeling demonstrate that continuous dynamics can be learned efficiently in a \gls{simfree} manner, without requiring numerical integration. At first glance, these generative models seem quite different from Latent \glspl{sde}. Instead of generating data autoregressively, they produce a single data point through an iterative refinement process, gradually reconstructing corrupted data. However, similar to Latent \glspl{sde}, continuous diffusion models can be thought of as learning an \gls{sde} in a latent space. This connection is useful for understanding and developing \gls{sde} Matching.

Conventional diffusion models are defined through two processes: the forward (or noising) process and the reverse (or generative) process. The forward process takes a data point $x \in \R^D$ and perturbs it over time by injecting noise. This dynamic can be expressed as an \gls{sde} with a linear drift term and state-independent diffusion term:
\begin{align}
    d z_t &= f(t) z_t d t + g(t) d \w,
    \label{eq:diff_f_sde}
\end{align}
where $f : [0, 1] \mapsto \R$, $g : [0, 1] \mapsto \R_+$, and $q(z_0|x) \approx \d(z_0 - x)$. Due to the linearity of \cref{eq:diff_f_sde}, the conditional marginal distribution is available in closed form $q(z_t|x) = \N(z_t; \a_t x, \s_t^2 I)$, where $\a_t, \s_t$ are determined by $f(t), g(t)$. The functions $f(t)$ and $g(t)$ are typically chosen to ensure that $q(z_1|x) \approx \N(z_1;0, I)$. The generative process then reverses this transformation, starting from the prior $p(z_1) = \N(z_1;0, I)$ and following the (marginal) reverse \gls{sde}:
\begin{align}
    d z_t = [ f(t) z_t - g^2(t) \n_{z_t} \log q(z_t) ] d t + g(t) d \bw.
    \label{eq:diff_marginal_f_rsde}
\end{align}
Here, $\bw$ denotes a standard Wiener process, where time flows backward. Diffusion models approximate this reverse process by learning the score function $\n_{z_t} \log q(z_t)$ using a denoising score matching loss:
\begin{align}
    \underset{u(t) q(z_t|x)}{\E} \left[ \big\| s_\th(z_t, t) - \n_{z_t} \log q(z_t|x) \big\|_2^2 \right],
    \label{eq:diff_loss}
\end{align}
where $u(t)$ represents a uniform distribution over the interval $[0, 1]$, and $s_\th: \R^d \times [0, 1] \mapsto \R^d$ is a learnable approximation of the score function.

A key property of the denoising score matching objective is that it can be learned in a \gls{simfree} manner. Due to the simplicity of the forward process, instead of integrating the \gls{sde} in \cref{eq:diff_f_sde}, we can directly sample from $q(z_t|x)$ and estimate the expectation in \cref{eq:diff_loss} using the Monte Carlo method. This property distinguishes diffusion models from Latent \glspl{sde}. However, as we will see diffusion models are in fact a special case of Latent \glspl{sde}.

To demonstrate this connection we: (1) derive the reverse \gls{sde} for the conditional forward process~(\cref{eq:diff_f_sde}), and (2) invert the time direction for the entire model, mapping the time interval $[0,1]$ into $[1,0]$.

First, the reverse \gls{sde} for the forward process can be obtained from the Fokker–Planck equation, similar to the derivation of the marginal reverse \gls{sde} in \cref{eq:diff_marginal_f_rsde}.

Second, after inverting the time direction, we obtain:
\begin{align}
    d z_t = [ f(t) z_t + g^2(t) \n_{z_t} \log q(z_t|x) ] d t + g(t) d \w.
    \label{eq:diff_f_rsde}
\end{align}

Except for the change in time direction, the generative process remains unchanged:
\begin{align}
    d z_t = [ f(t) z_t + g^2(t) s_\th(z_t, t) ] d t + g(t) d \w.
    \label{eq:diff_r_sde}
\end{align}

From this perspective, diffusion models can be seen as a special case of Latent \glspl{sde} with a specific form of the \gls{ptproc} and only a single observation $x$ at time step $t=1$. Moreover, substituting this parameterization of the prior and reverse processes into the Latent \gls{sde} objective in \cref{eq:background_loss}, we find that it corresponds to a reweighted denoising score matching objective:
\begin{align}
    & \underset{q(\Z|x)}{\E} \bigg[ \int_0^1 \frac{1}{ 2 g^2(t) } \Big\| 
        \cancel{ f(t) z_t } + g^2(t) s_\th(z_t, t) \nonumber \\
        & \quad \quad \quad - \cancel{ f(t) z_t } - g^2(t) \n_{z_t} \log q(z_t|x) 
    \Big\|_2^2 dt \bigg] + C \\
    & = \underset{u(t) q(z_t|x)}{\E} \left[ \frac{g^2(t)}{2} \big\| s_\th(z_t, t) - \n_{z_t} \log q(z_t|x) \big\|_2^2 \right] + C.
\end{align}

This result is particularly meaningful since it is well known that denoising score matching~(\cref{eq:diff_loss}) corresponds to a reweighted variational bound on the likelihood of diffusion models \cite{song2021maximum}.

This connection between diffusion models and Latent \glspl{sde}, along with the fact that diffusion models can be trained in a \gls{simfree} manner, motivates us to extend and develop a \gls{simfree} framework for training Latent \glspl{sde}.

\section{SDE Matching}
\label{sec:method}

Diffusion models can be trained in a \gls{simfree} manner because they do not require full simulation of the noising process. Instead, they allow direct sampling of latent variables $z_t$ from the marginal distribution $q(z_t|x)$. In contrast, Latent \glspl{sde} are often trained using a  \gls{ptproc} parameterized by a general-form \gls{sde}. This renders the marginal posterior distribution $q_\ph(z_t|\X)$ intractable, which prevents \gls{simfree} training.

To address this limitation, we propose \gls{sde} Matching -- a framework for \gls{simfree} training of Latent \gls{sde} models. The key idea behind \gls{sde} Matching is to parameterize the \gls{ptproc}' conditional \gls{sde} via a learnable function $F_\ph(\e, t, \X)$ that directly defines the marginal distributions $q_\ph(z_t|\X)$. This design inherently allows direct sampling of latent variables without numerical integration. Importantly, \gls{sde} Matching simplifies only the training procedure while leaving the generative dynamics of the Latent \gls{sde} \gls{prproc} fully flexible.

\subsection{Generative Model}

In \gls{sde} Matching, the \gls{prproc} is identical to the standard Latent \gls{sde} model~(see \cref{sec:background}):
\begin{align}
    d z_t &= h_\th(z_t, t) d t + g_\th(z_t, t) d \w.
    \label{eq:prior_sde}
\end{align}

In general, the functions $h_\th$ and $g_\th$ may be parameterized using neural networks of arbitrary form. However, the choice of parameterization involves some trade-offs, which we will discuss later.

Observations $x_{t_i}$ are likewise assumed to be sampled conditionally independently from the likelihood distributions ${x_{t_i} \sim p_\th(x_{t_i}|z_{t_i})}$ for each time step $t_i$.

\subsection{Posterior Process}
\label{sec:posterior_process}

Instead of defining the \gls{ptproc} through a conditional \gls{sde}~(\cref{eq:background_posterior_sde}) and then deriving its analytical solution, we propose an alternative approach. We first define the conditional marginal distribution of latent variables $q_\ph(z_t|\X)$ for each $t$, and then derive the corresponding conditional \gls{sde} with the desired marginal distributions.

We construct the posterior process by following  three steps:
\begin{enumerate}
    \item Define the marginal distribution implicitly through a parameterized function of noise;
    \item Construct a conditional \gls{ode} that satisfies the target marginals;
    \item Transition from the \gls{ode} to a \gls{sde} while preserving the marginal distributions.
\end{enumerate}

A detailed description of each step with proofs is provided in \cref{app:posterior}.

\textbf{Posterior Marginal Distribution.} We define the conditional marginal distribution $q_\ph(z_t|\X)$ implicitly. First, we introduce a function that, given time $t$ and observations $\X$, transforms noise $\e$ into a latent variable $z_t$:
\begin{align}
    z_t = F_\ph(\e, t, \X),
    \label{eq:f}
\end{align}
where $\e \sim q(\e) = \N(\e; 0, I)$. This implicitly defines the conditional distribution of latent variables $q_\ph(z_t|\X)$. Although, in general, we do not have an explicit form for this distribution, the above definition inherently enables efficient sampling. The specific parameterization of $F_\ph$ and the dependence of $z_t$ on observations $\X$ are user design choices.

\textbf{Conditional \gls{ode}.} We assume that $F_\ph$ is differentiable with respect to $\e$ and $t$, and invertible with respect to $\e$. Fixing observations $\X$ and noise $\e$, and varying $t$ from $0$ to $1$, results in a smooth trajectory from $z_0$ to $z_1$. Differentiating these trajectories over time yields a velocity field that generates the conditional distribution $q_\ph(z_t|\X)$, the conditional \gls{ode}:
\begin{align}
    d z_t & = \barf_\ph(z_t, t, \X) d t, 
    \quad \textrm{where} \quad \label{eq:posterior_ode} \\
    \barf_\ph(z_t, t, \X) & = \left. \frac{\p F_\ph(\e, t, \X)}{\p t} \right|_{\e=F_\ph^{-1}(z_t, t, \X)}.
\end{align}

Thus, if we sample $z_0 \sim q_\ph(z_0|\X)$ and solve \cref{eq:posterior_ode} up to time $t$, we obtain a sample $z_t \sim q_\ph(z_t|\X)$.

The time derivative of $F_\ph$ can be computed efficiently using forward-mode differentiation in modern frameworks such as PyTorch \cite{paszke2017automatic} or JAX \cite{jax2018github}.

\textbf{Conditional \gls{sde}}. Finally, we derive the conditional \gls{sde} that defines the \gls{ptproc}.

Given access to both the conditional \gls{ode} and the score function $\n_{z_t} \log q_\ph(z_t|\X)$, we follow \citet{song2021scorebased} and define a conditional \gls{sde} with marginal distributions $q_\ph(z_t|\X)$ as follows:
\begin{align}
    d z_t = f_{\th,\ph}(z_t, t, \X) &d t  + g_\th(z_t, t) d \w, \quad \textrm{where} 
    \label{eq:posterior_sde} \\
    f_{\th,\ph}(z_t, t, \X) &= 
    \label{eq:posterior_sde_drift}
         \barf_\ph(z_t, t, \X) \\
        + &\frac{1}{2} g_\th(z_t, t) g^\top_\th(z_t, t) \n_{z_t} \log q_\ph(z_t|\X) \nonumber \\
        + &\frac{1}{2} \n_{z_t} \cdot \left[ g_\th(z_t, t) g^\top_\th(z_t, t) \right]. \nonumber
\end{align}
Here, the diffusion term $g_\th$ is the matrix-valued function from the \gls{prproc}. It is crucial that the \gls{ptproc} has the same diffusion term as the \gls{prproc} to ensure that the variational bound is finite. Notably, $g_\th$ affects only the distribution of trajectories ${z(t)}_{t \in [0, 1]}$, while the marginal distributions $q_\ph(z_t|\X)$ do not depend on $g_\th$.

We would like to emphasis that the correspondence between the posterior \gls{sde} in \cref{eq:posterior_sde} and its marginals $q_\ph(z_t|\X)$ is exact, not approximate.

Evaluating the drift term in \cref{eq:posterior_sde} presents several challenges. First, it requires access to the conditional score function $\n_{z_t} \log q_\ph(z_t|\X)$, which can be computationally expensive. However, for $F_\ph$ functions that enable efficient log-determinant computation of the Jacobian matrix, the score function can be computed efficiently (e.g., functions linear in $\e$ or RealNVP architectures \cite{dinh2017density, kingma2018glow}). The calculation of this score function is further discussed in \cref{app:parameterization}.

The second challenge involves computing the last term in \cref{eq:posterior_sde_drift}, which can be expensive in high-dimensional settings. However, it can be computed efficiently if for example the diffusion term $g_\th$ is of the form $\sigma_\th(z_t,t) \Gamma_\th(t)$ for scalar-valued $\sigma_\th$ and matrix-valued $\Gamma_\th$. Another option is if $g_\th$ is a diagonal matrix, where each element $g_\th(z_t, t)_{k,k}$ depends only on the $k$-th coordinate of $z_t$. Furthermore, if $g_\th$ does not depend on the state $z_t$, this term vanishes. Notably, while this structure limits flexibility, it is identical to constraints in memory-efficient implementations of standard Latent \gls{sde} training \cite{li2020scalable}, which means that \gls{sde} Matching does not introduce additional restrictions.

\subsection{Optimization}

\begin{algorithm}[!t]
\caption{Training of \gls{sde} Matching}
\label{alg:training}
\begin{algorithmic}
    \REQUIRE $\X$, $p_\th(z_0)$, $h_\th$, $g_\th$, $p_\th(x_t|z_t)$, $F_\ph$
    \FOR{learning iterations}
        \STATE {\color{gray} $\#$ Calculation of prior loss $\Lprior$}
        \STATE $\e_0 \sim \N(0, I)$
        \STATE $z_0 = F_\ph(\e_0, 0, \X)$
        \STATE $\Lprior = \KL \big( q_\ph(z_0|\X) \| p_\th(z_0) \big)$
        \STATE {\color{gray} $\#$ Calculation of diffusion loss $\Ldiff$}
        \STATE $t \sim u(t)$
        \STATE $\e_t \sim \N(0, I)$
        \STATE $z_t = F_\ph(\e_t, t, \X)$
        \STATE $r_{\th,\ph}(z_t, t, \X) = g^{-1}_\th(z_t, t) \left( h_\th(z_t, t) -  f_{\th,\ph}(z_t, t, \X) \right)$
        \STATE $\Ldiff = \frac{1}{2} \left\| r_{\th,\ph}(z_t, t, \X) \right\|_2^2$
        \STATE {\color{gray} $\#$ Calculation of reconstruction loss $\Lrec$}
        \STATE $i \sim u(i)$
        \STATE $\e_{t_i} \sim \N(0, I)$
        \STATE $z_{t_i} = F_\ph(\e_{t_i}, t_i, \X)$
        \STATE $\Lrec = - N \log p_\th(x_{t_i}|z_{t_i})$
        \STATE {\color{gray} $\#$ Optimization}
        \STATE $\L = \Lrec + \Ldiff + \Lprior$
        \STATE Gradient step on $\th$ and $\ph$ w.r.t. $\L$
    \ENDFOR
\end{algorithmic}
\end{algorithm}

Like standard Latent \gls{sde} training, \gls{sde} Matching optimizes the parameters $\th$ and $\ph$ end-to-end by optimizing the same variational objective. However, the training procedure of \gls{sde} Matching can be made \gls{simfree} by leveraging the above \gls{ptproc} and rewriting the objective from \cref{eq:background_loss} as follows:
\begin{align}
    -\log p_\th(\X) & \leq \L = \Lprior + \Ldiff + \Lrec 
    \label{eq:loss} \\
    \Lprior & = \KL \big( q_\ph(z_0|\X) \| p_\th(z_0) \big) \\
    \Ldiff & = \underset{u(t)q_\ph(z_t|\X)}{\E} \left[ \frac{1}{2} \left\| r_{\th,\ph}(z_t, t, \X) \right\|_2^2 \right] 
    \label{eq:l_diff} \\
    \Lrec & = N \underset{u(i)q_\ph(z_{t_i}|\X)}{\E} \left[ - \log p_\th(x_{t_i}|z_{t_i}) \right],
\end{align}
where $u(t)$ represents a uniform distribution over the interval $[0, 1]$, $u(i)$ represents a uniform distribution over the set ${1, \dots, N}$, and $r_{\th,\ph}(z_t, t, \X)$ satisfies:
\begin{align}
    g_\th(z_t, t) r_{\th,\ph}(z_t, t, \X) = h_\th(z_t, t) -  f_{\th,\ph}(z_t, t, \X).
\end{align}

The objective in \cref{eq:loss} consists of three terms: the prior loss $\Lprior$, the diffusion loss $\Ldiff$, and the reconstruction loss $\Lrec$. The optimization of $\Lprior$ with respect to parameters $\th$ can be omitted if we are not interested in unconditional sampling or in learning the initial prior $p_\th(z_0)$—for example, if we are solely interested in forecasting. Otherwise, this term can still be optimized efficiently.

The other two terms, $\Ldiff$ and $\Lrec$, which previously required the numerical integration of the conditional \gls{sde}, can now be estimated more efficiently. Specifically, for each term, we can first sample a timestep $t$, then sample $z_t \sim q_\ph(z_t|\X)$, and finally evaluate the corresponding function. Importantly, \gls{sde} Matching allows the estimation of $\Ldiff$ and $\Lrec$ with an arbitrary number of samples evaluated in parallel. It also enables inference of the \gls{ptproc}~(\cref{eq:posterior_sde}) as a special case. However, we found that using a single sample was sufficient for our experiments. We summarize the training procedure in \cref{alg:training}.

Note that while we use the functions $\barf_\ph$ (\cref{eq:posterior_ode}) and $f_{\th,\ph}$ (\cref{eq:posterior_sde}) to describe the conditional \gls{sde} in the \gls{ptproc}, these functions are ultimately defined by the reparameterization function $F_\ph$~(\cref{eq:f}) and $g_\th$, which are the functions we actually parameterize.

The full Bayesian treatment is also possible as a straightforward extension of the bound in \cref{eq:loss} for a given prior with a separate variational approximation for $\th$.

We would like to emphasize that the \gls{sde} Matching objective is exactly the same variational bound as used in the adjoint sensitivity method. The key difference lies in the estimation strategy for the objective. We discuss the tightness of the variational bound in \cref{eq:loss} in \cref{app:theory} and empirically validate convergence to the true posterior for a linear \gls{sde} in \cref{app:convergence}.

\subsection{Sampling}

Unconditional sampling at inference- or test-time for \gls{sde} Matching follows the same procedure as sampling from a Latent \gls{sde}. First, we sample an initial latent variable $z_0 \sim p_\th(z_0)$ and then integrate the unconditional \gls{sde} in \cref{eq:prior_sde} using any off-the-shelf \gls{sde} solver. Then, at each desired time step $t$, we project the latent states $z_t$ to the data space by sampling $x_t$ from the distribution $p_\th(x_t|z_t)$.

However, if we have access to partial observations $\X = {x_{t_1}, x_{t_2}, \dots, x_{t_N}}$ and are interested in forecasting for $t > t_N$, we can instead first sample the latent state $z_{t_N}$ from $q_\ph(z_{t_N}|\X)$ and then integrate the \gls{prproc} dynamics using this sample as initial condition. This follows because the Latent \gls{sde} is Markovian with observations that are conditionally independent given the latent state. Importantly, because \gls{sde} Matching provides direct access to an approximation to the posterior (smoothing) marginals it can be used for \gls{simfree} sampling (inference) of the latent state, whereas conventional parameterization of the \gls{ptproc} would require simulating the conditional \gls{sde} up to time step $t_N$ first.

Similarly, interpolation can be performed by inferring only the \gls{ptproc} dynamics. Alternatively, for more general conditioning events and steering one could leverage \gls{smc} methods \citep{naesseth2019elements,chopin2020introduction,wu2023practical} for conditional sampling, but a detailed investigation of these approaches is beyond the scope of this work.

Due to the strong connection between \gls{sde} Matching and diffusion models—and given that diffusion models have demonstrated exceptional performance in deterministic sampling—it is natural to ask whether it is possible to sample deterministically from \gls{sde} Matching. Indeed, the \gls{prproc} and the initial distribution $p_\th(z_0)$ together define a sequence of marginal probability distributions $p_\th(z_t)$. While it is possible to learn an \gls{ode} corresponding to $p_\th(z_t)$, it is important to clarify that solving this \gls{ode} does not necessarily produce meaningful time-series samples.

Although the \gls{ode} preserves marginal distributions, it alters the joint distribution of latent variable trajectories. The same applies to the observed variables: while the marginal distributions remain correct, the joint distribution becomes distorted. Therefore, using an \gls{ode} for deterministic sampling of time series data does not generally preserve the correct structure of the generated trajectories.

\section{Experiments}
\label{sec:experiments}

\begin{figure}[!t]
    \centering
    \parbox{.25\textwidth}{
        \begin{subfigure}{\linewidth}
            \includegraphics[width=\textwidth]{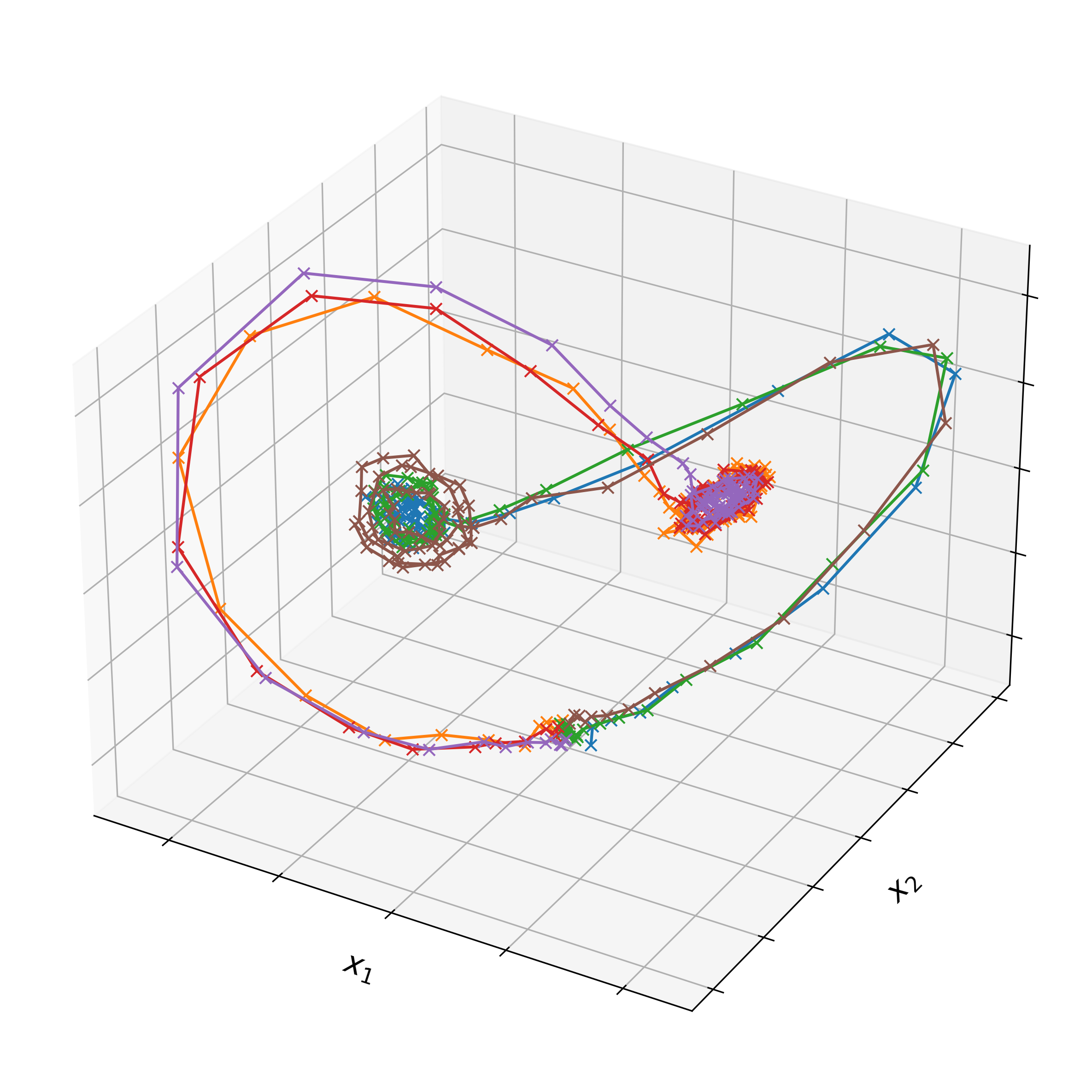}
            \caption{Training data distribution}
        \end{subfigure}
    }
    \parbox{.235\textwidth}{
        \begin{subfigure}{\linewidth}
            \includegraphics[width=\textwidth]{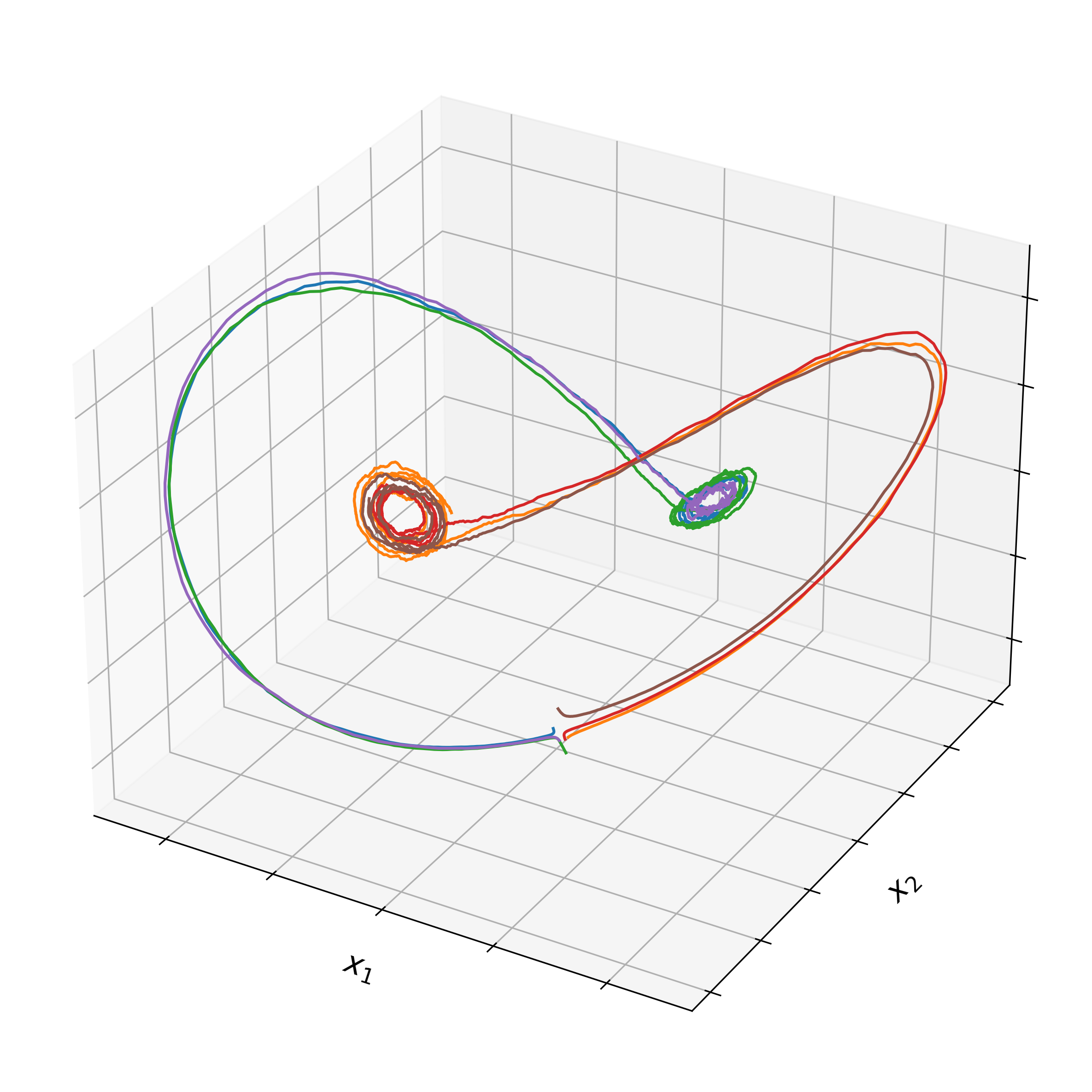}
            \caption{Adjoint sensitivity method}
        \end{subfigure}
    }
    \parbox{.235\textwidth}{
        \begin{subfigure}{\linewidth}
            \includegraphics[width=\textwidth]{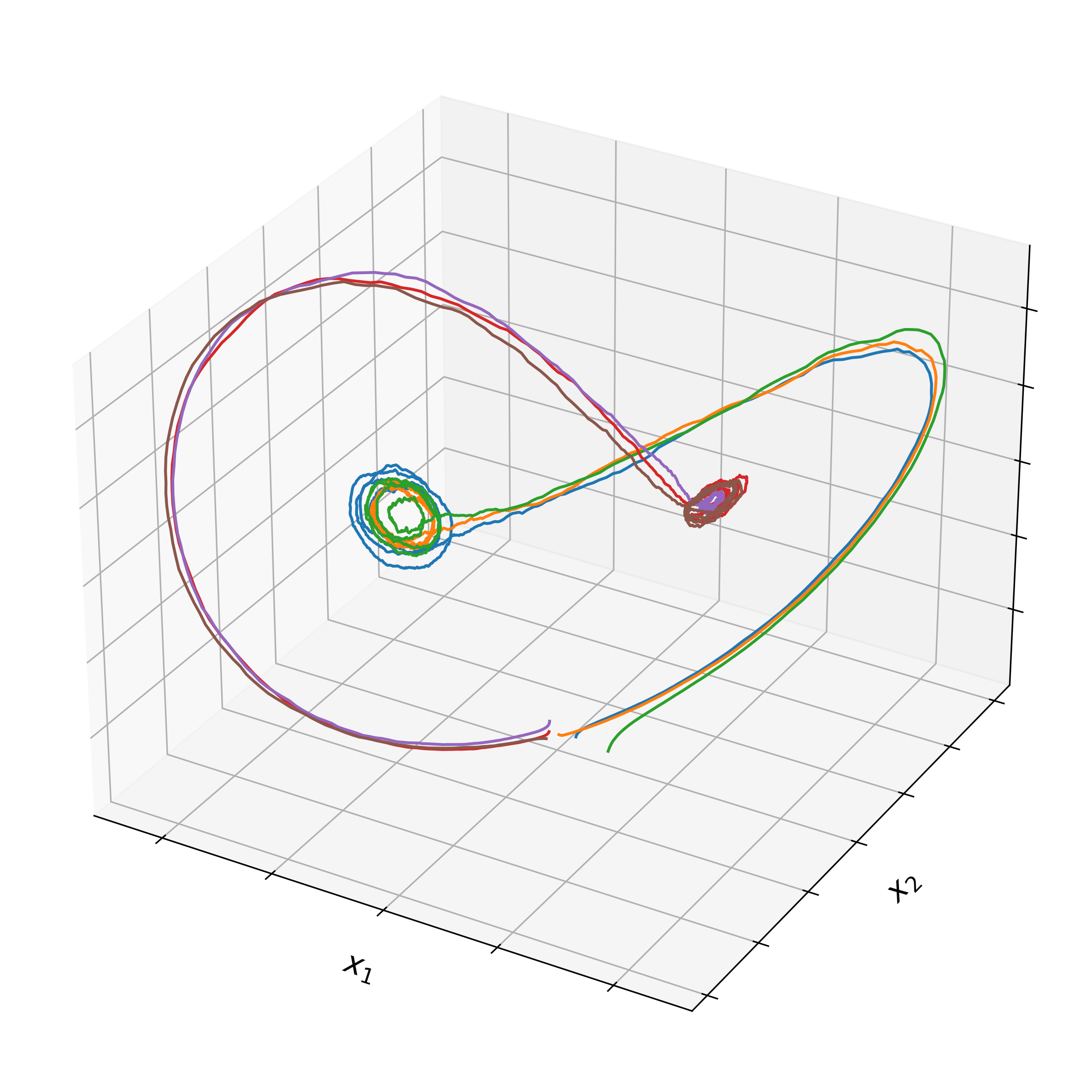}
            \caption{\gls{sde} Matching}
        \end{subfigure}
    }
    \caption{Training data distribution and learned dynamics from a 3D stochastic Lorenz attractor. Results for Latent \gls{sde} trained with adjoint sensitivity method (\textit{bottom left}) and \gls{sde} Matching (\textit{bottom right}). Both methods successfully learn the underlying dynamics.}
    \label{fig:synthetic}
\end{figure}

\begin{figure}[!t]
    \centering
    \centerline{\includegraphics[width=.49\textwidth]{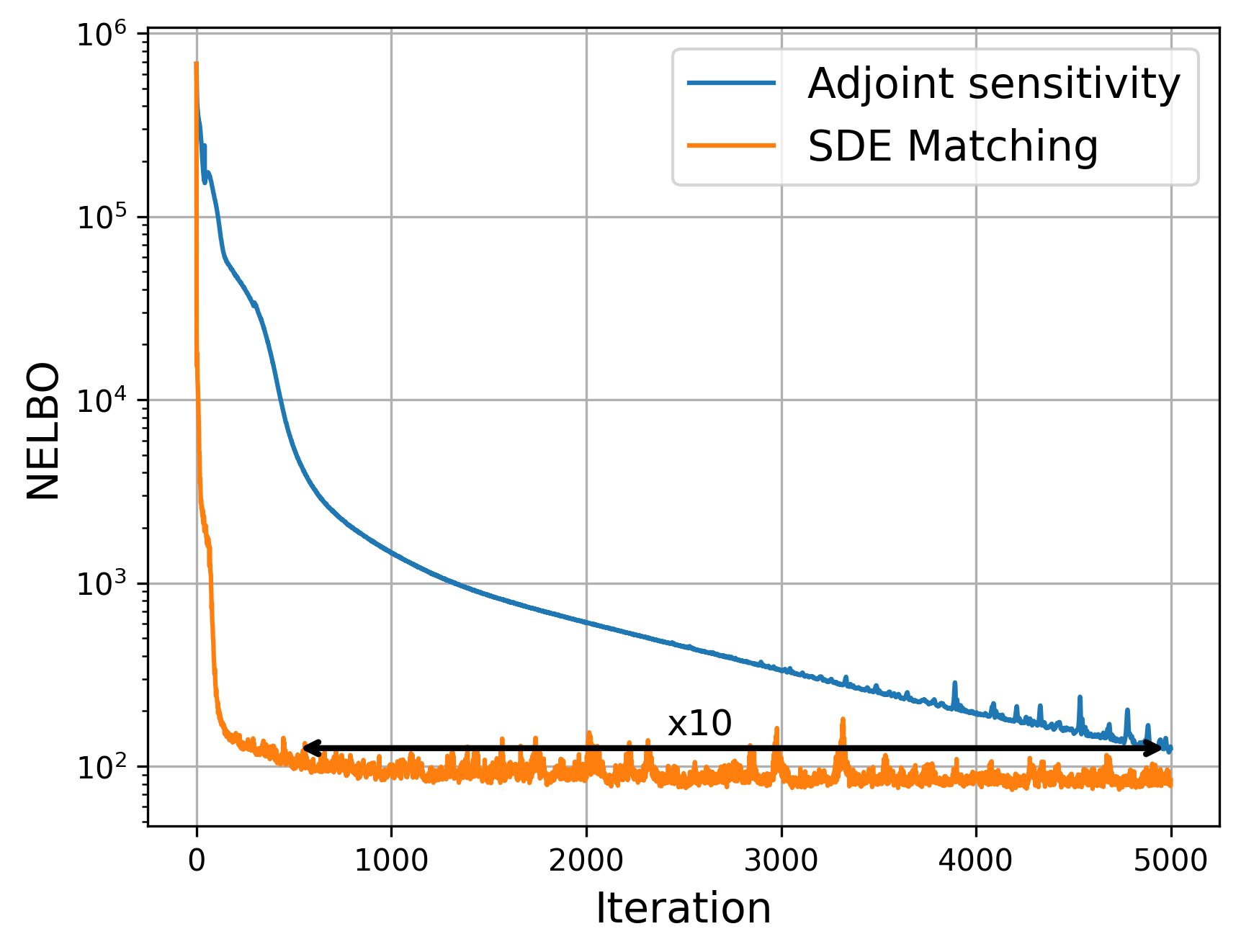}}
    \caption{\gls{nelbo}~(\cref{eq:loss}) objective with respect to iteration step on 3D stochastic Lorenz attractor data. Results for Latent \gls{sde} trained with \textcolor[HTML]{1F77B4}{adjoint sensitivity method} and \textcolor[HTML]{ff7f0e}{\gls{sde} Matching}. \gls{sde} Matching demonstrates approximatly $10$ times faster convergence in terms of iterations.}
    \label{fig:convergence}
\end{figure}

\gls{sde} Matching is a \gls{simfree} framework for training Latent \gls{sde} models. Therefore, the goal of this section is to demonstrate that \gls{sde} Matching indeed enables computationally efficient training of Latent \glspl{sde}. Our objective is not to outperform all other Latent \gls{sde}-based approaches for time series modeling. Instead, \gls{sde} Matching is compatible with existing extensions and can be combined with them for further improvements.

To this end, we provide various experiments on both synthetic and real data. In all experiments, we use the same hyperparameters for both the conventional parameterization of Latent \gls{sde} and the parameterization described in \cref{sec:method}. Additional details on parameterization and training is provided in \cref{app:implementation}. When using the \gls{sde} Matching training procedure, we jointly optimize the parameters of the generative model $\th$ in the prior $p_\th(z_0)$, drift $h_\th(z_t, t)$, diffusion $g_\th(z_t, t)$, and observation model $p_\th(x_t|z_t)$, as well as the parameters $\ph$ of the posterior reparameterization function $F_\ph$. All parameters are optimized jointly by minimizing the variational bound in \cref{eq:loss}. We do not apply annealing or additional regularization techniques during training.

The experiments demonstrate that \gls{sde} Matching achieves comparable or better accuracy than simulation-based Latent \gls{sde} training, while significantly reducing computational complexity. Not only does \gls{sde} Matching reduce the computational cost of each training iteration, the parameterization of the \gls{ptproc} also leads to faster convergence in terms of the number of iterations for further gains compared to alternatives. The findings suggest that \gls{sde} Matching enhances the scalability of Latent \glspl{sde} to longer and higher-dimensional time series.

The code is available at \url{https://github.com/GrigoryBartosh/sde_matching}.

\subsection{Synthetic Datasets}
\label{sec:lorenz}

For the synthetic data experiment, we consider the 3D stochastic Lorenz attractor process from \citet{li2020scalable} with identical observation generation procedure. 

As shown in \cref{fig:synthetic}, both models—one trained using the adjoint sensitivity method and the other with \gls{sde} Matching—successfully learned the underlying dynamics. However, \gls{sde} Matching required only a single evaluation of the drift term in the \gls{ptproc} for each iteration, whereas the adjoint sensitivity method required 100 simulation steps for this experiment. In terms of absolute runtime, a single training iteration with \gls{sde} Matching was approximately five times faster.

Moreover, as demonstrated in \cref{fig:convergence}, \gls{sde} Matching leads to faster convergence of the model. We attribute this to the parameterization of the \gls{ptproc} through the reparameterization function $F_\ph$~(\cref{eq:f}). We hypothesize that since this function directly models the marginal distributions of the \gls{ptproc}, the model can more efficiently learn compared to integrating conditional \glspl{sde}.

The combined, per-iteration and convergence, runtime speed-up is over $50\times$ compared to the adjoint sensitivity approach for this experiment.

To evaluate robustness with respect to time horison, we compare the gradient norms of the objective function with respect to model parameters for both SDE Matching and the adjoint sensitivity method. When integrating over time horizons $T \in {1, 2, 5, 10}$, the adjoint sensitivity method yields the following $\log_{10}$-based gradient norms:  $6.26 \pm 0.14$, $6.95 \pm 0.20$, $7.91 \pm 0.23$ and $8.82 \pm 0.28$. These results demonstrate that the gradient norms grow exponentially with the time horizon for adjoint sensitivity methods.

In contrast, SDE Matching maintains a stable gradient norm of $4.92 \pm 0.24$ (in $\log_{10}$ scale) across all time horizons, indicating significantly greater stability for long time series modeling.

In \cref{app:grads} we provide additional experiments demonstrating robustness of the \gls{sde} Matching training procedure.

\subsection{Motion Capture Dataset}
\label{sec:mocap}

To validate \gls{sde} Matching on real-world data, we follow \citet{li2020scalable,course2024amortized} and evaluate it on the motion capture dataset from \citet{gan2015deep}. This dataset consists of motion recordings from subject 35 walking, represented as 50-dimensional time series with 300 observations each. The dataset is split into 16 training sequences, 3 validation sequences, and 4 test sequences, following the preprocessing method from \citet{wang2007gaussian}.

We use the same hyperparameters as \citet{li2020scalable}, including setting the latent state dimensionality to 6 and keeping the number of training iterations unchanged. In \cref{tab:mocap}, we report the average performance on the training set over 10 models trained from different random seeds. \gls{sde} Matching achieves similar performance to the adjoint sensitivity method while being significantly more computationally efficient.

\begin{table}[!t]
    \caption{Test MSE on motion capture dataset. We report an average performance based on $10$ models trained with deferent random seeds and $95\%$ confidence interval based on t-statistic. The first section of the table contains \gls{simbased} approaches and the second part contains \gls{simfree} approaches.  $^\dagger$results from \cite{yildiz2019ode2vae}, $^*$from \cite{li2020scalable} and $^\ddagger$from \cite{course2024amortized}.}
    \label{tab:mocap}
    \centering
    \begin{tabular}{ll}
        \toprule
        Model & Test MSE $\downarrow$ \\
        \midrule
        DTSBN-S \cite{gan2015deep} & $34.86 \pm 0.02^\dagger$ \\
        npODE \cite{heinonen2018learning} & $22.96^\dagger$ \\
        NeuralODE \cite{chen2018neural} & $22.49 \pm 0.88^\dagger$ \\
        ODE$^2$VAE \cite{yildiz2019ode2vae} & $10.06 \pm 1.4^\dagger$ \\
        ODE$^2$VAE-KL \cite{yildiz2019ode2vae} & $8.09 \pm 1.95^\dagger$ \\
        Latent \gls{ode} \cite{rubanova2019latent} & $5.98 \pm 0.28^*$ \\
        Latent \gls{sde} \cite{li2020scalable} & $\mathbf{4.03 \pm 0.2}^*$ \\
        \midrule
        ARCTA \cite{course2024amortized} & $7.62 \pm 0.93^\ddagger$ \\
        \gls{sde} Matching (ours) & $\mathbf{4.50 \pm 0.32}$ \\
        \bottomrule
    \end{tabular}
\end{table}

\begin{figure}[!t]
    \centering
    \parbox{.49\textwidth}{
        \begin{subfigure}{\linewidth}
            \includegraphics[width=\textwidth]{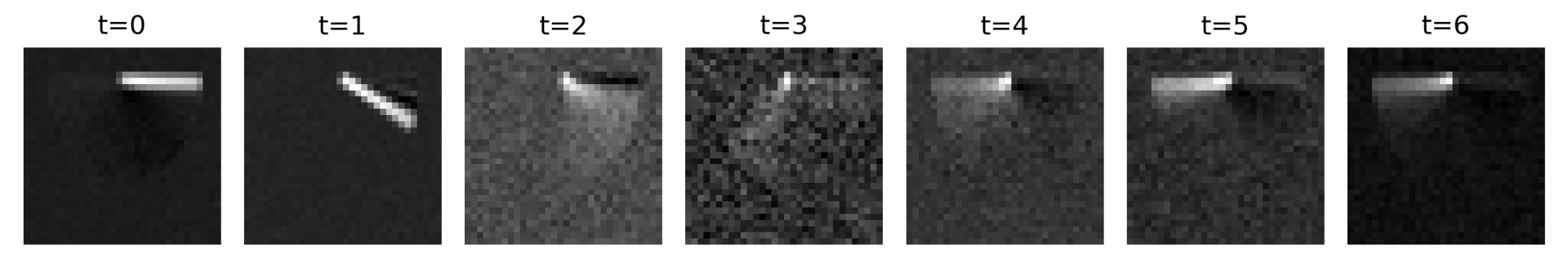}
            \caption{Adjoint sensitivity method}
        \end{subfigure}
    }
    \parbox{.49\textwidth}{
        \begin{subfigure}{\linewidth}
            \includegraphics[width=\textwidth]{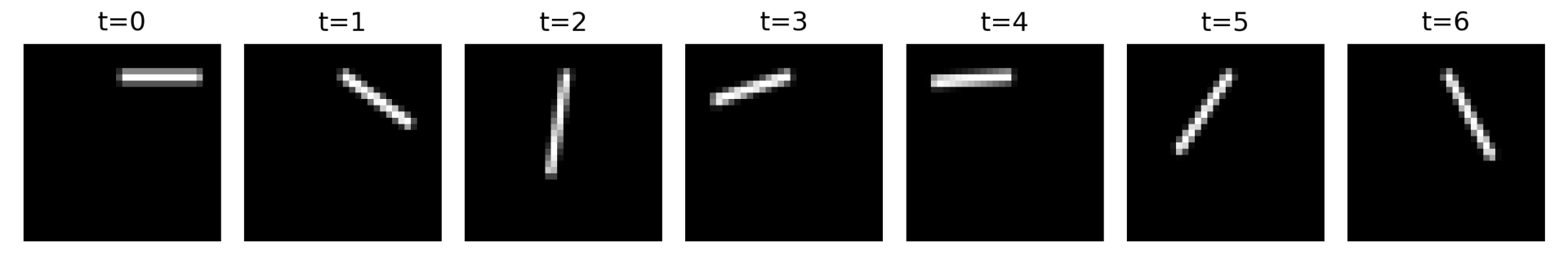}
            \caption{SDE Matching}
        \end{subfigure}
    }
    \caption{Qualitative comparison of generated video sequences depicting a moving pendulum. Samples generated by the model trained with adjoint sensitivity method and SDE Matching with same number of training iterations.}
    \label{fig:video}
\end{figure}

\gls{sde} Matching also outperforms the other \gls{simfree} approach, ARCTA \citet{course2024amortized}, which can be seen as a special case of \gls{sde} Matching. Notably, ARCTA requires drawing around 100 latent samples and evaluations of \glspl{sde} per batch, whereas \gls{sde} Matching requires only one. We attribute this result to the fact that ARCTA employs a less flexible prior and posterior parameterization, where the volatility function $g_\th$ is state-independent and distribution of latent states $q_\ph(z_t|\X)$ depends heavily on observations close to time $t$. This dependence makes it challenging to learn long-term dependencies and limits the ability to propagate learning signals effectively.

An extended discussion on the importance of state-dependent volatility functions $g_\th$, along with experiments on the state-dependent stochastic Lotka-Volterra system, is provided in \cref{app:diffusion_term}.

\subsection{Video Dataset}
\label{sec:video}

Finally, we evaluate the performance of SDE Matching in a high-dimensional setting by modeling sequences of $32 \times 32$ images—i.e. video data—depicting a moving pendulum.

Using an identical number of training iterations ($20$k), SDE Matching successfully learns the underlying dynamics in approximately 20 minutes on a single GPU. In contrast, the adjoint sensitivity method requires around 30 hours for the same number of iterations and still fails to accurately capture the system's dynamics.

Visualizations of the generated sequences are provided in \cref{fig:video}.

\section{Related Work}
\label{sec:related_work}

Differential equations are a widely recognized technique for modeling continuous dynamics. Recently they have seen a surge of interest in machine learning after the introduction of Neural \glspl{ode} by \citet{chen2018neural} and \citet{rubanova2019latent}. Neural \glspl{ode} demonstrated strong performance in time-series modeling compared to recurrent neural networks, particularly for irregularly sampled time series. However, Neural \glspl{ode} have several limitations.

First, they rely on adjoint sensitivity methods for training, which requires numerical integration of gradients and backpropagation through the solutions. Aside from being computationally expensive and difficult to parallelize on modern hardware, \glspl{ode} can sometimes suffer from numerical instabilities \cite{lea2000sensitivity}. Second, \gls{ode}-based models inherently encode all uncertainty into the initial conditions as the dynamics is completely deterministic. This approach is inadequate for modeling fundamentally stochastic processes, especially when observations are sparse or uninformative.

Traditionally, Monte Carlo methods have been used to train \glspl{sde} \citep{movellan2002monte,beskos2006exact,van2017bayesian}. Another line of work \cite{archambeau2007gaussian, archambeau2007variational} has focused on variational inference techniques for inferring the Latent \glspl{sde}'s training objective. Discretization is used by \citet{ryder2018black} to enable more expressive variational approximations, and \citet{ha2018adaptive} draw the connection to stochastic optimal control for more flexibility. 

More recently, \citet{li2020scalable} extended the adjoint sensitivity method from \glspl{ode} to \glspl{sde}. Unlike Neural \glspl{ode}, Latent \glspl{sde} are better suited for modeling inherently stochastic and chaotic processes and are more robust to data shifts.

Many studies have proposed modifications to objective functions, introduced regularized dynamics, and improved computational efficiency and numerical stability for both Neural \glspl{ode} \cite{kelly2020learning, finlay2020train, kidger2021hey} and Latent \glspl{sde} \cite{kidger2021efficient}. However, most of these methods still rely on backpropagation through the numerical solutions of differential equations limiting their scalability. In \citet{verma2024variational} the authors instead rely on a site-based Gaussian approximation and suggest a fixed-point training algorithm.

Closest to our work is \citet{course2024amortized} that develops a method for training Latent \glspl{sde} in a \gls{simfree} manner. This approach is restricted to \glspl{ptproc} with Gaussian marginals and does not support state-dependent diffusion terms in the \gls{prproc}. Additionally, this method assumes that the latent state in the \gls{ptproc} depends only on a few temporally closest observations. As a result, for effective optimization it requires drawing approximately 100 latent samples,  rather than 1, per batch during training. This method can be seen as a special case of \gls{sde} Matching with the corresponding restrictions on the parameterization of the \gls{prproc} and \gls{ptproc}.

Recently, \citet{zhang2024trajectory} proposed a \gls{simfree} technique for time-series modeling. However, this approach learns dynamics directly in the original data space rather than in a latent space. Furthermore, to model non-Markovian processes the authors condition the generative dynamics on a fixed set of past observations, making simulations more computationally expensive and limiting the model's ability to capture long-range dependencies. This can be viewed as a special case of \gls{sde} Matching, where the latent space is constructed by concatenation of the most recent observations.

Another line of research that explores learning stochastic dynamics in a \gls{simfree} manner is diffusion models \cite{ho2020denoising, song2021scorebased}. As we demonstrated in \cref{sec:diffusion}, conventional diffusion models can be seen as a special case of Latent \glspl{sde} with only a single observation. The key property that enables efficient training in conventional diffusion models is their simple and fixed noising process, i.e., the \gls{ptproc}.

Recently \citet{singhal2023where,bartosh2023neural, nielsen2024diffenc, sahoo2024diffusion, bartosh2024neural, du2024doob} have shown that the noising process in diffusion models can be more complex, and even learnable, while still preserving the \gls{simfree} property. These approaches can also be viewed as special cases of \gls{sde} Matching by the re-interpretation of diffusion models as Latent \glspl{sde}.

\section{Limitations and Future Work}
\label{sec:limitations}

The \gls{simfree} properties of \gls{sde} Matching come with certain trade-offs. First, \gls{sde} Matching parameterizes the \gls{ptproc} through the function $F_\ph(\e, t, \X)$~(\cref{eq:f}). This function must not only be smooth and invertible with respect to $\e$, it must also provide access to the conditional score function $\n_{z_t} \log q_\ph(z_t|\X)$~(\cref{eq:posterior_sde_drift}). These requirements restrict the flexibility of $F_\ph$ and, consequently, the posterior distribution approximation. Nevertheless, to the best of our knowledge, \gls{sde} Matching offers the most flexible parameterization that enables \gls{simfree} access to latent samples $z_t \sim q_\ph(z_t|\X)$ of the conditional \gls{sde} \gls{ptproc}.

Another limitation is the computational cost of evaluating the \gls{ptproc}' \gls{sde} drift term for high-dimensional $g_\th(z_t, t)$~(\cref{eq:prior_sde}) of general form, as discussed in \cref{sec:posterior_process}. However, it is worth reiterating that other memory-efficient training of Latent \glspl{sde} faces this same limitation.

\gls{sde} Matching does not constrain the flexibility of the \gls{prproc}, which defines the generative model. Nevertheless, developing methods that allow even more flexible parameterization of $F_\ph$ and diffusion terms $g_\th$ remains a promising direction for future research. Additionally, understanding the impact of this flexibility in modeling latent dynamics versus dynamics in the original data space is an interesting open question.

We believe that thanks to its \gls{simfree} properties, \gls{sde} Matching has the potential to scale Latent \gls{sde}-based modeling to high-dimensional applications like audio and video generation and many other applications where they were previously infeasible. It could potentially also contribute to the development of \gls{simfree} methods for training Latent \gls{ode} models and  learning policies in reinforcement learning. We leave these directions for future work.

\section{Conclusion}
\label{sec:conclusion}

We introduced \gls{sde} Matching, a \gls{simfree} framework for training Latent \glspl{sde}. By leveraging insights from score-based generative models, we formulated a method that eliminates the need for costly numerical simulations while maintaining the expressiveness of Latent \glspl{sde}. Our approach directly parameterizes the marginal posterior distributions, enabling efficient training and conditional inference.

Through both synthetic and real-world experiments, we demonstrated that \gls{sde} Matching achieves performance comparable to or better than existing adjoint sensitivity-based methods, while significantly reducing computational complexity. The ability to estimate latent trajectories without solving high-dimensional \glspl{sde} makes our method particularly suitable for large-scale time-series and sequence modeling.

Despite its advantages, \gls{sde} Matching introduces trade-offs in terms of posterior parameterization flexibility. Nevertheless, to the best of our knowledge, \gls{sde} Matching proposes the most flexible parameterization of the \gls{ptproc} that allows \gls{simfree} sampling of latent variables. Exploring more expressive function classes for posterior distributions and further extending \gls{sde} Matching are promising directions for future work.

Furthermore, our approach provides a foundation for scaling Latent \glspl{sde} to previously infeasible sequential domains such as audio and video. We believe that the scalability and efficiency of \gls{sde} Matching will enable broader applications in scientific modeling, finance, and healthcare, where structured uncertainty modeling is critical. Future research may also explore extensions to \gls{simfree} training of Latent \glspl{ode} and connections to reinforcement learning and stochastic optimal control.

\section*{Impact Statement}

This paper presents work whose goal is to advance the field of Machine Learning. There are many potential societal consequences of our work, none which we feel must be specifically highlighted here.

\bibliography{bibliography}
\bibliographystyle{icml2025}

\newpage
\appendix
\onecolumn
\section{Detailed Description of the Posterior Process}
\label{app:posterior}

\subsection{Marginal Distribution}
\label{app:posteriormarginal}
The posterior marginal distribution is defined by a transformation $F_\ph$ from noise $\e$ for each $t$ and $\X$:
\begin{align}
    z_t = F_\ph(\e, t, \X), \quad \e \sim q(\e) = \N(\e; 0, I).
\end{align}
Assuming that $F_\ph$ is invertible and differentiable in $\e$ for each $t$ and $\X$. Then, we can use the standard change of variables result for distributions \citep{rudin2006real,bogachev2007measure} to derive the density of $z_t$
\begin{align}
    q_\ph(z_t|\X) &= q(\e) \Big|_{\e = F_\ph^{-1}(z_t, t, \X)} \cdot \left| \frac{\p F^{-1}_\ph(z_t, t, \X)}{\p z_t} \right|,
    \label{eq:app_marginal}
\end{align}
where $F^{-1}_\ph$ is the inverse of $F_\ph$ and $|\cdot|$ denotes the absolute value of the determinant. This is a well-known result of normalizing flows \citep{papamakarios2021normalizing} applied to our setting.

\subsection{Conditional \gls{ode}}
\label{app:posteriorode}
We leverage results based on \emph{flows}, solutions to \glspl{ode}, inspired by \citet{bilovs2021neural,lee2012introduction,bartosh2024neural} to turn a flow $F_\ph$ into its corresponding generating \gls{ode} $\barf_\ph$. 
\begin{proposition}
    Let $F_\ph: \R^D \times \R \times \setX \mapsto \R^D$ be a smooth function such that $F_\ph(\cdot, t, \X)$ is invertible $\forall t, \X$. Then, $F_\ph(\cdot, t, \X)$ is a \emph{flow} with the infinitesimal generator
    \begin{align}
        \frac{d F_\ph(\e, t, \X)}{dt},
    \end{align}
    which is a smooth vector field $f$ such that
    \begin{align}
         \frac{\p F_\ph(\e, t, \X)}{\p t} &= f(F_\ph(\e, t, \X),t,\X).
    \end{align}
    \label{prop:flowode}
\end{proposition}
\begin{proof}
    The result is a consequence of the Fundamental Theorem on Flows \citep[Theorem 9.12]{lee2012introduction} for the flow defined by $F_\ph(\cdot, t, \X)$.
\end{proof}
Using \cref{prop:flowode} with $\e = F_\ph^{-1}(z_t, t, \X)$, identifying that $F_\ph(F_\ph^{-1}(z_t, t, \X),t,\X) = z_t$, lets us construct the conditional \gls{ode} whose solution matches those of the flow $z_t = F_\ph(\e, t, \X)$:
\begin{align}
    d z_t & = \barf_\ph(z_t, t, \X) d t, 
    \quad \textrm{where} \quad \label{eq:app_posterior_ode} \\
    \barf_\ph(z_t, t, \X) & = \left. \frac{\p F_\ph(\e, t, \X)}{\p t} \right|_{\e=F_\ph^{-1}(z_t, t, \X)}.
\end{align}
Initializing using $z_0 \sim q_\phi(z_0|\X)$ and solving the above conditional \gls{ode} in \cref{eq:app_posterior_ode} ensures $z_t \overset{d}{=} F_\ph(\e, t, \X)$ for $\e \sim q(\e)$, where $\overset{d}{=}$ denotes equal in distribution.

\subsection{Conditional \gls{sde}}
\label{app:posteriorsde}
To derive the conditional \gls{sde} we leverage a result by \citet{song2021scorebased} that relates the marginal distributions of an \gls{sde} with its corresponding (probability flow) \gls{ode}, restated in \cref{prop:odesde} for convenience.

\begin{proposition}
    The marginal distributions $p_t(z_t)$ of the \gls{sde}
    \begin{align}
         d z_t &= h(z_t, t) d t + g(z_t, t) d \w,
    \end{align}
    are, under suitable conditions on the drift and diffusion terms $h$ and $g$, identical to the distributions induced by the \gls{ode}
    \begin{align}
        d z_t &= \left(h(z_t, t) -\frac{1}{2} g(z_t, t) g^\top(z_t, t) \n_{z_t} \log p_t(z_t)  -\frac{1}{2} \n_{z_t} \cdot \left[g(z_t, t) g^\top(z_t, t)\right]\right) d t.
        \label{eq:app_probflowode}
    \end{align}
    \label{prop:odesde}
\end{proposition}
\begin{proof}
    See \citet[Appendix D.1]{song2021scorebased}.
\end{proof}
Note that this equivalence lets us easily find the corresponding (probability flow) \gls{ode}. However, we can equally well use it in the reverse order with a known \gls{ode} to construct the corresponding \gls{sde} by selecting a diffusion term $g$ and using \cref{eq:app_probflowode} to compute the drift term $h$.

For the \gls{prproc} diffusion term $g_\th(z_t,t)$ with velocity field \cref{eq:app_posterior_ode} we apply \cref{prop:odesde} to find the conditional \gls{sde} with marginals $q_\ph(z_t|\X)$
\begin{align}
    d z_t = f_{\th,\ph}(z_t, t, \X) &d t  + g_\th(z_t, t) d \w, \quad \textrm{where} 
    \label{eq:app_posterior_sde} \\
    f_{\th,\ph}(z_t, t, \X) &= 
    \label{eq:app_posterior_sde_drift}
         \barf_\ph(z_t, t, \X) + \frac{1}{2} g_\th(z_t, t) g^\top_\th(z_t, t) \n_{z_t} \log q_\ph(z_t|\X) + \frac{1}{2} \n_{z_t} \cdot \left[ g_\th(z_t, t) g^\top_\th(z_t, t) \right].
\end{align}
This result follows by simply identifying terms in \cref{prop:odesde} with our specific case
\begin{align*}
    h(z_t,t) &\equiv f_{\th,\ph}(z_t, t, \X), \\
    g(z_t,t) &\equiv g_\th(z_t,t).
\end{align*}

\subsection{Parameterization}
\label{app:parameterization}

As discussed in \cref{sec:posterior_process}, evaluating the posterior \gls{sde} (\cref{eq:posterior_sde}) requires access to the conditional score function $\n_{z_t} \log q_\ph(z_t|\X)$. Since we parameterize the posterior marginals $q_\ph(z_t|\X)$ implicitly through an invertible transformation $F_\ph$ (\cref{eq:f}) of a random variable $\e \sim q(\e)$ into $z_t$, the score function can in general be computed using \cref{eq:app_marginal}:
\begin{align}
    & \n_{z_t} \log q_\ph(z_t|\X) = \n_{z_t} \left[ \log q(\e) \Big|_{\e = F_\ph^{-1}(z_t, t, \X)} + \log \left| \frac{\p F^{-1}_\ph(z_t, t, \X)}{\p z_t} \right| \right].
    \label{eq:app_score}
\end{align}

The first term in \cref{eq:app_score} represents the log-density of the noise distribution, which is straightforward to compute. The second term is the log-determinant of the Jacobian matrix of the inverse transformation. If this Jacobian is available, the score function can be efficiently computed using automatic differentiation tools.

In this work, we parameterize the function $F_\ph$ as follows:
\begin{align}
    F_\ph(\e, t, \X) = \mu_\ph(\X, t) + \s_\ph(\X, t) \e, \quad \textrm{and} \quad F_\ph^{-1}(z_t, t, \X) = \s_\ph^{-1}(\X, t) \big( z_t - \mu_\ph(\X, t) \big)
    \label{eq:app_f}
\end{align}
where $\s_\ph$ is a matrix valued function.

Substituting this parameterization into \cref{eq:app_score}, the score function simplifies to:
\begin{align}
    \n_{z_t} \log q_\ph(z_t|\X) = \left[ \s_\ph(\X, t) \s_\ph^\top(\X, t) \right]^{-1} \big( \mu_\ph(\X, t) - z_t \big) = \left. - \s_\ph^{-\top}(\X, t) \e \right|_{\e=F_\ph^{-1}(z_t, t, \X)}.
\end{align}

The \gls{sde} Matching framework allows for other, potentially more flexible, parameterizations of the function $F_\ph$. However, we leave the exploration of such alternatives for future research.

\section{Theoretical Considerations}
\label{app:theory}

In this section, we provide an extended discussion on the connection between the prior and posterior processes. Our primary goal is to address the question: under what conditions does the variational bound in \cref{eq:loss} become tight and match the true likelihood of the model?

The variational bound becomes tight when the approximate posterior distribution exactly matches the true posterior. In the case of the Latent \gls{sde}, this means that the \gls{ptproc} must match the true posterior of the \gls{prproc}.

The prior process~(\cref{eq:prior_sde}) is given by:
\begin{align}
    d z_t &= h_\th(z_t, t) d t + g_\th(z_t, t) d \w.
\end{align}

Given a series of observations $\X$, the true posterior process can be derived via Doob’s $h$-transform \citep{mider2021continuous,park2024amortized}, yielding:
\begin{align}
    d z_t &= \Big[ \underbrace{ h_\th(z_t, t) + g_\th(z_t, t) g^\top_\th(z_t, t) \n_{z_t} \log p_\th \left( \X_t | z_t \right) }_{h_\th(z_t, t, \X)} \Big] d t + g_\th(z_t, t) d \w,
\end{align}
where $\X_t = \{x_s : x_s \in \X, s \geq t\}$ denotes the set of future observations from time $t$ onward.

We can also consider the corresponding deterministic process that follows the posterior marginals $p_\th(z_t|\X)$
\begin{align}
    \frac{d z_t}{d t} 
    &= \barh_\th(z_t, t, \X) \\
    &= h_\th(z_t, t) + g_\th(z_t, t) g^\top_\th(z_t, t) \left[ \n_{z_t} \log p_\th \left( \X_t | z_t \right) - \frac{1}{2} \n_{z_t} \log p_\th \left( z_t | \X \right) \right] - \frac{1}{2} \n_{z_t} \cdot \left[ g_\th(z_t, t) g^\top_\th(z_t, t) \right]
    \label{eq:app_prior_ode}
\end{align}

Now, consider the approximate \gls{ptproc}~(\cref{eq:posterior_sde}) of the form:
\begin{align}
    d z_t &= f_{\th, \ph}(z_t, t, \X) d t + g_\th(z_t, t) d \w.
\end{align}
The approximate process matches the true posterior process if the drift terms are equal: $f_{\th, \ph}(z_t, t, \X) \equiv h_\th(z_t, t, \X)$ or equivalently, if the corresponding deterministic processes~(\cref{eq:posterior_ode,eq:app_prior_ode}) match: $\barf_{\th, \ph}(z_t, t, \X) \equiv \barh_\th(z_t, t, \X)$. In terms of the reparameterization function $F_\ph(\e, t, \X)$~(\cref{eq:f}) that defines the approximate \gls{ptproc}, tthe variational bound becomes tight, and the approximate \gls{ptproc} matches the true posterior, if $F_\ph$ learns the solution trajectories of the \gls{ode} defined by $\barh_\th(z_t, t, \X)$~(\cref{eq:app_prior_ode}).

Unfortunately, as is typical in variational inference, the true posterior is intractable. Nevertheless, its analytical form provides valuable insight into the behavior that the approximate posterior should approximate. Therefore, with a sufficiently flexible reparameterization function $F_\ph(\e, t, \X)$, the \gls{sde} Matching approach should, in principle, be capable of learning the true \gls{ptproc} and the variational bound becomes equal to the negative log-marginal likelihood.

\section{Additional Results}
\label{app:additional}

\subsection{Convergence to Likelihood}
\label{app:convergence}

To empirically validate the tightness of the ELBO in SDE Matching, as discussed in \cref{app:theory}, we set up a linear stochastic system of the form:
\begin{align}
    d z_t &= F(t) z_t dt + L(t) d\w \\
    x_t &= H(t) z_t + r_t,
\end{align}
where $r_t \sim \N(0, R(t))$. For this system, the true log-marginal likelihood can be computed using the Kalman filter \citep{sarkka2019applied}.

In this experiment, we use simple coefficients: $F(t) = -t$, $L(t) = t$, $H(t) = 1$, and $R(t) = 0.01$. The model is parameterized identically to our experiment on the 3D stochastic Lorenz attractor (\cref{sec:lorenz}). As shown in \cref{fig:lin_nll}, the SDE Matching ELBO closely converges to the true log-marginal likelihood. Furthermore, similar to our results in \cref{sec:lorenz}, SDE Matching achieves significantly faster convergence than the adjoint sensitivity method and requires approximately $20 \times$ less time per training iteration.

\subsection{Robustness of Gradients}
\label{app:grads}

We present additional experiments to assess the robustness of SDE Matching.

Using the setup from \cref{app:convergence}, we compare the gradient norms and variances over training time. As shown in \cref{fig:lin}, SDE Matching consistently exhibits both lower gradient norms and reduced variance compared to the adjoint sensitivity method.

\subsection{Importance of State-Dependent Diffusion Term}
\label{app:diffusion_term}

We also present an example that highlights the importance of state-dependent diffusion term (volatility), a feature not supported by the previous simulation-free approach ARCTA \citep{course2024amortized}.

To demonstrate this, we design an experiment based on the stochastic Lotka–Volterra system with dynamics:
\begin{align}
    d x &= (\a x + \beta x y ) dt + \s x d \w \\
    d y &= (\delta xy + \gamma y ) dt + \s y d \w.
\end{align}

We follow the setup from \citet{course2024amortized} (Appendix G.1) with volatility coefficient $\sigma = 0.15$ and use the same parameterization as in \cref{sec:lorenz}.

We then apply SDE Matching to train models with both state-independent and state-dependent volatility functions. As illustrated in \cref{fig:lotka_volterra}, the model with state-independent volatility fails to accurately capture the correct trajectory structure, whereas the state-dependent model succeeds.

State-dependant volatility is often a relevant property of many dynamical systems in e.g. finance \citep{oksendal2003stochastic}, biology \citep{vadillo2019comparing}, neuroscience \citep{elgazzar2024generative}, and causal inference \cite{peters2022causal}, where Latent SDEs can be applied for modelling. The Lotka–Volterra system which we use in this experiment is an example of such a biological system. In Neural SDEs, state-dependant volatility also leads to dynamics that is more robust to distribution shift \citep{oh2024stable}. Finally, the state-dependant volatility is an important contributing factor in our MOCAP experiment (\cref{sec:mocap}), where the state-dependent SDE Matching and \citet{li2020scalable} both perform much better than the state-independent ARCTA \citep{course2024amortized}. Therefore, we believe that, together with more general marginals of the posterior process through normalising flows, enabling simulation-free training with state-dependant volatility is also an important contribution of SDE Matching.

\begin{figure}[!t]
    \centering
    \parbox{.3\textwidth}{
        \begin{subfigure}{\linewidth}
            \includegraphics[width=\textwidth]{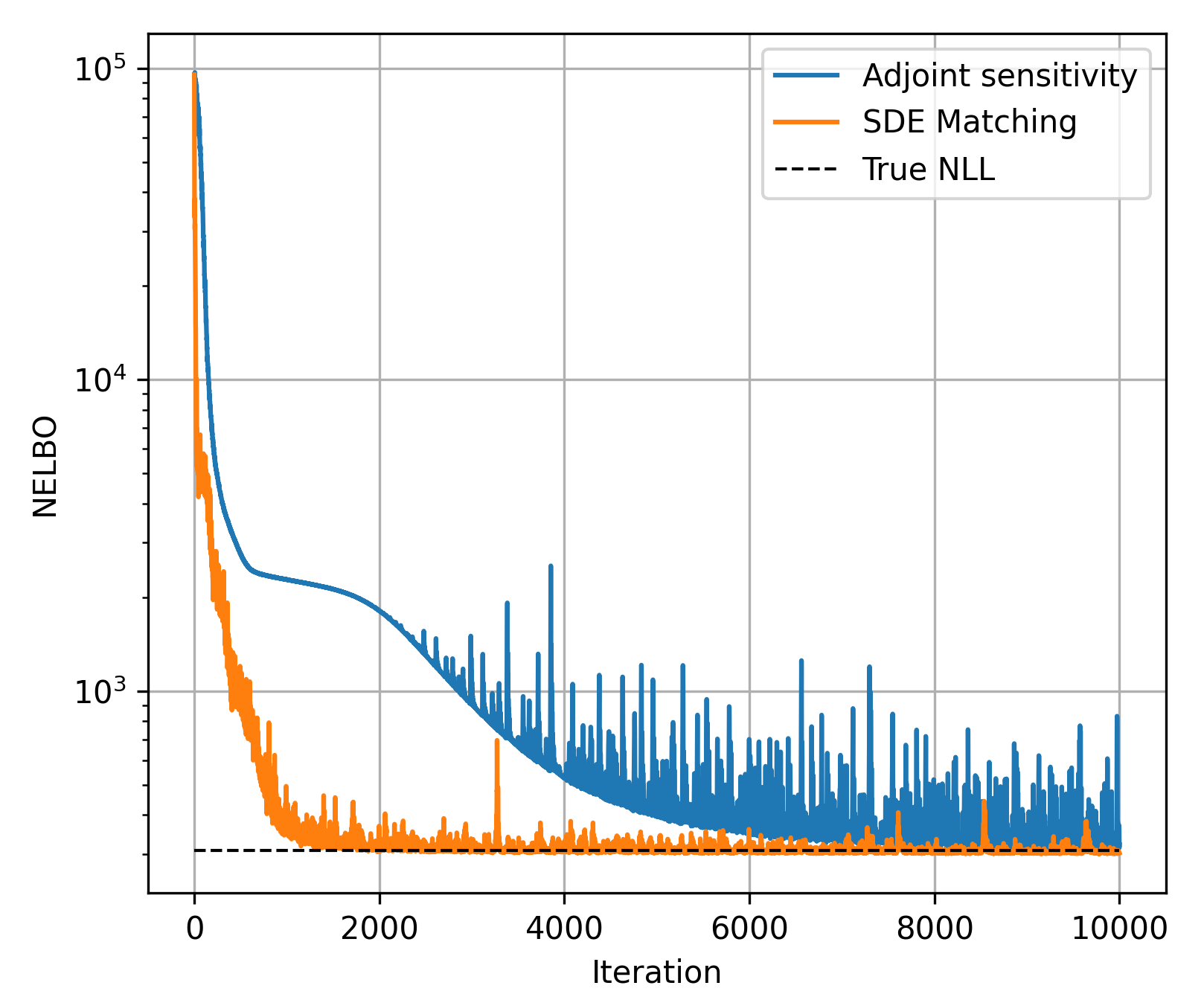}
            \label{fig:lin_nll}
            \caption{Convergence of NELBO}
        \end{subfigure}
    }
    \parbox{.3\textwidth}{
        \begin{subfigure}{\linewidth}
            \includegraphics[width=\textwidth]{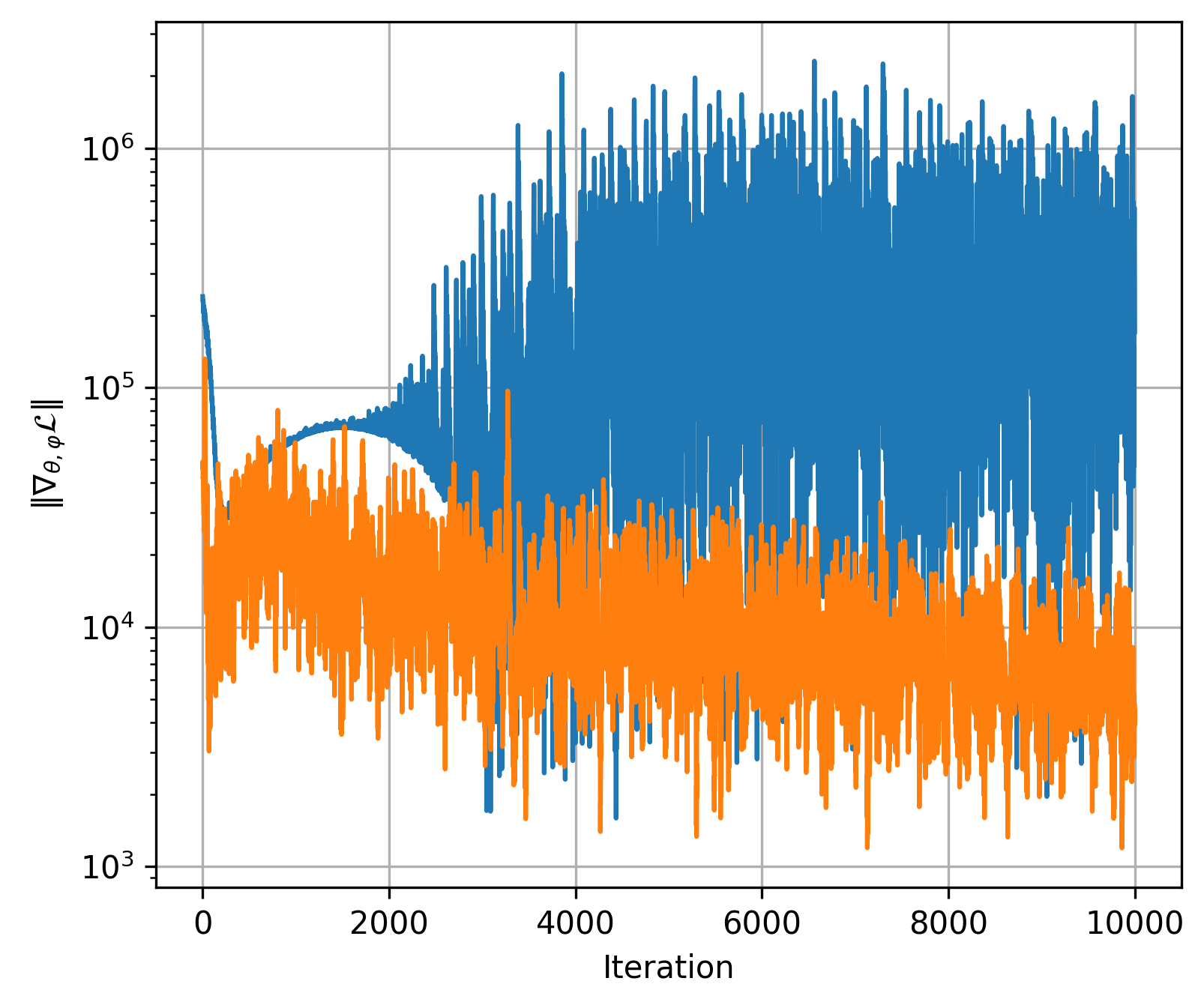}
            \caption{Norm of gradient}
        \end{subfigure}
    }
    \parbox{.3\textwidth}{
        \begin{subfigure}{\linewidth}
            \includegraphics[width=\textwidth]{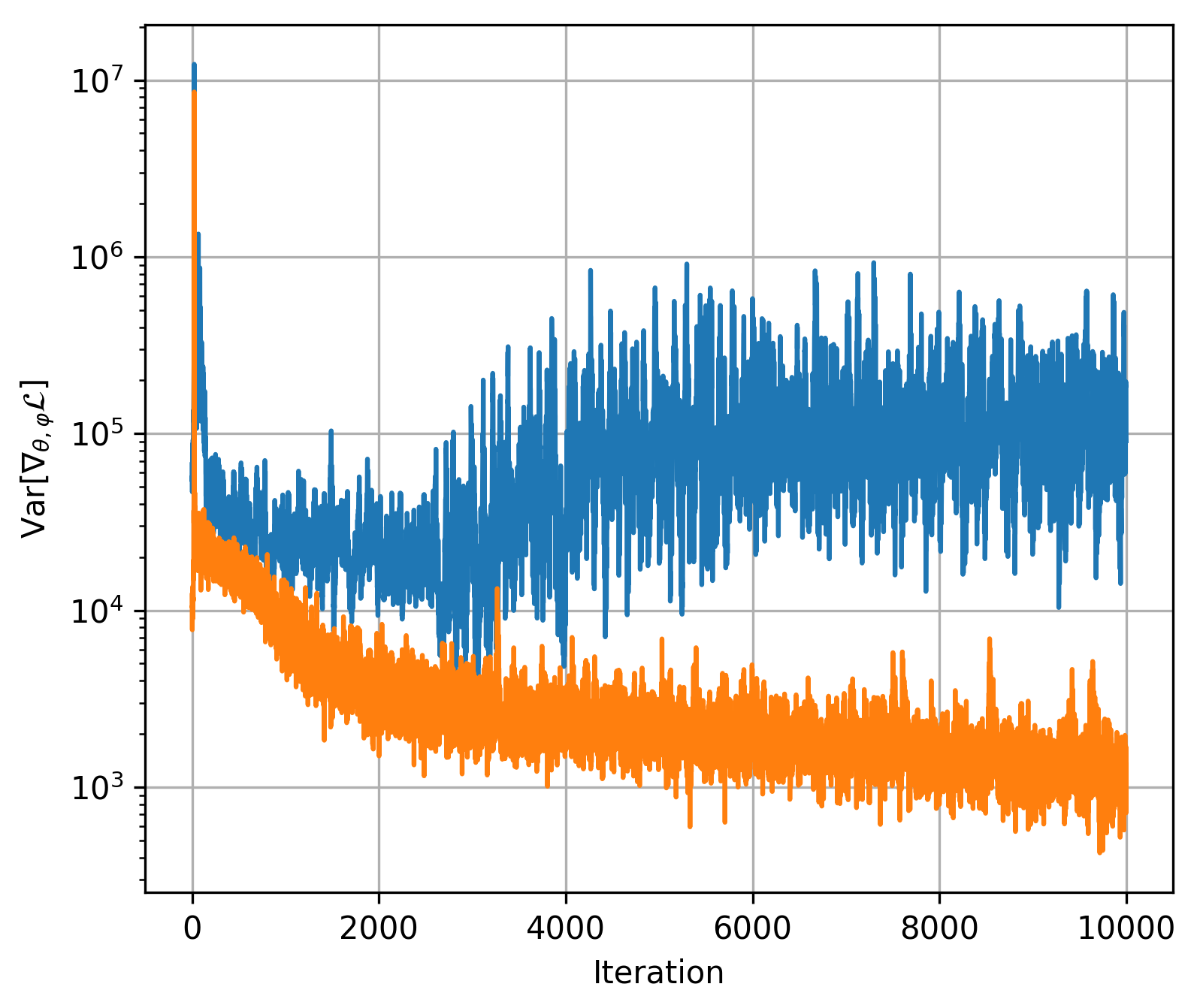}
            \caption{Variance of gradients}
        \end{subfigure}
    }
    \caption{Linear system training dynamics for Latent \gls{sde} trained with \textcolor[HTML]{1F77B4}{adjoint sensitivity method} and \textcolor[HTML]{ff7f0e}{\gls{sde} Matching}.}
    \label{fig:lin}
\end{figure}

\begin{figure}[!t]
    \centering
    \centerline{\includegraphics[width=\textwidth]{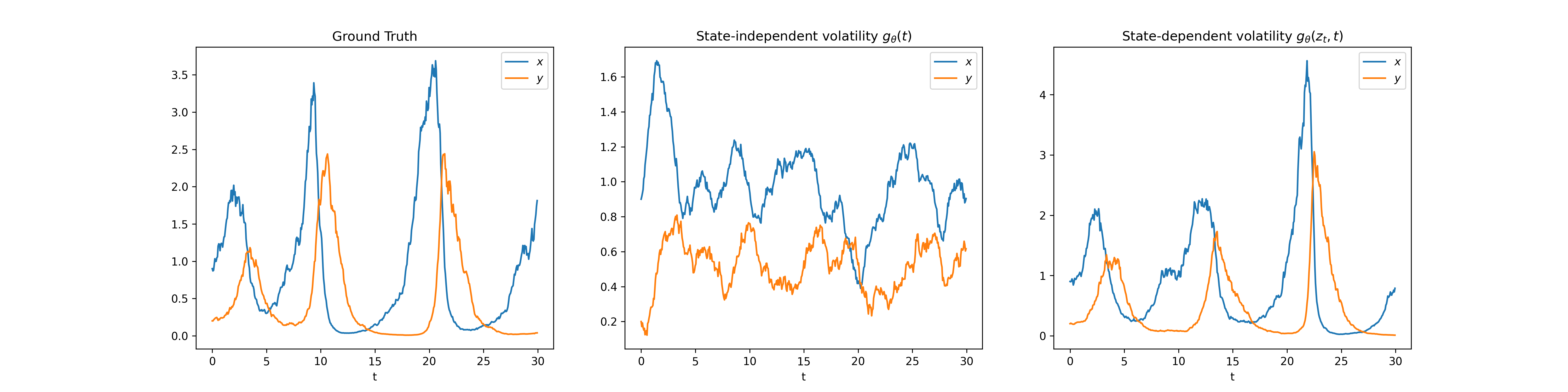}}
    \caption{Lotka–Volterra system, learned trajectories for SDE Matching with state-independent and state-dependent volatility.}
    \label{fig:lotka_volterra}
\end{figure}

\section{Implementation Details}
\label{app:implementation}

For the experiments presented in \cref{sec:experiments}, we follow the setup from \citet{li2020scalable} for both the 3D stochastic Lorenz attractor and the motion capture dataset. We use the same parameterizations and hyperparameters for the \gls{prproc}, including the observation models $p_\th(x_t|z_t)$ and the coordinate-wise diagonal diffusion term $g_\th(z_t, t)$.

We make a slight modification to the parameterization of the \gls{ptproc}. First, we do not use an encoder, i.e., a separate network that defines the initial conditions of the \gls{ptproc}. Second, \citet{li2020scalable} propose using a time-reversal GRU layer \cite{cho2014learning} to aggregate information from observations into what they refer to as a context. This context is then used to predict the drift term of the posterior \gls{sde} via an MLP. To maintain consistency with \citet{cho2014learning}, we use the same hyperparameters for the GRU and MLP. However, instead of predicting the drift term, we use the MLP to predict the functions $\mu_\ph$ and $\log \s_\ph$~(\cref{eq:app_f}) based on the context computed from all observations and the time step $t$. We use diaganal parameterization $\s_\ph$.

For the moving pendulum video experiment, we follow the setup from \citet{course2024amortized} for both data generation and the parameterization of the \gls{prproc}. For the \gls{ptproc}, we use the same parameterization as in previous experiments, but replace the MLP with a convolutional encoder from \citet{course2024amortized} to process image observations.

For optimization, we use the same training hyperparameters, including the Adam optimizer \cite{kingma2014adam} and the same number of training iterations. However, we do not apply any additional regularization, reweighting, or annealing techniques.

We believe that there exists much more efficient ways to parameterize the \gls{ptproc} for \gls{sde} Matching. However, in this work, our primary goal was to provide a fair comparison with the standard Latent \gls{sde} training approach rather than focusing on specific design choices such as prior and posterior model architectures.

\end{document}